\documentclass[letterpaper, 10 pt, conference]{ieeeconf}  
\IEEEoverridecommandlockouts                              

\usepackage{url}

\usepackage{mathrsfs}
\usepackage{color}
\usepackage{subcaption}
\usepackage{times}
\usepackage{epsfig}
\usepackage{graphicx}
\usepackage{tabularx}
\usepackage{amsmath}
\usepackage{float}
\usepackage{relsize}
\usepackage{amssymb}
\usepackage[ruled,noend]{algorithm2e}
\usepackage{algpseudocode}
\usepackage{graphicx} 
\usepackage{tikz}
\usepackage{verbatim}
\usepackage{booktabs}
\usepackage{multirow}
\usepackage{hyphenat} 
\usepackage{pgf}
\usepackage{url}
\usepackage{thmtools}
\usepackage{makecell}
\usepackage{authblk}
\newtheorem{theorem}{Theorem}
\newtheorem{remark}{Remark}
\usepackage{xcolor}
\newtheorem{proofpart}{Proof Part}
\usepackage{xcolor}

\usepackage{hyperref}

\begin{document}

%
\title{\Large \bf DCL-Sparse: Distributed Range-only Cooperative Localization of Multi-Robots in Noisy and Sparse Sensing Graphs  	}
\author{Atharva Sagale \and  Tohid Kargar Tasooji \and Ramviyas Parasuraman
\thanks{School of Computing, University of Georgia, Athens, GA 30602, USA.  \\
Corresponding Author Email: {\tt\small ramviyas@uga.edu}  
}}

\maketitle              

\begin{abstract}
This paper presents a novel approach to range-based cooperative localization for robot swarms in GPS-denied environments, addressing the limitations of current methods in noisy and sparse settings. We propose a robust multi-layered localization framework that combines shadow edge localization techniques with the strategic deployment of UAVs. This approach not only addresses the challenges associated with nonrigid and poorly connected graphs but also enhances the convergence rate of the localization process. We introduce two key concepts: the S1-Edge approach in our distributed protocol to address the rigidity problem of sparse graphs and the concept of a powerful UAV node to increase the sensing and localization capability of the multi-robot system. Our approach leverages the advantages of the distributed localization methods, enhancing scalability and adaptability in large robot networks. 
We establish theoretical conditions for the new S1-Edge that ensure solutions exist even in the presence of noise, thereby validating the effectiveness of shadow edge localization. Extensive simulation experiments confirm the superior performance of our method compared to state-of-the-art techniques, resulting in up to 95\% reduction in localization error, demonstrating substantial improvements in localization accuracy and robustness to sparse graphs. This work provides a decisive advancement in the field of multi-robot localization, offering a powerful tool for high-performance and reliable operations in challenging environments.
\end{abstract}



\section{Introduction}

Multi-robot systems (MRS) \cite{parker2008multiple,rizk2019cooperative, latif2023} play a crucial role in various real-world applications, such as urban search and rescue, mining, exploration, and environmental monitoring. In scenarios like search and rescue missions in disaster-stricken areas, robots can be deployed to explore inaccessible environments, such as collapsed buildings, dense forests, or underwater cave systems, where human intervention is difficult or dangerous. Specifically, in urban search and rescue scenarios, limited infrastructure—such as the absence of GPS and reliable communication networks—makes teleoperation by humans impractical due to input lag and weak connectivity. Under such conditions, teams of robots must rely on distributed approaches for cooperative localization, as the lack of exact coordinate information in a global frame of reference poses a major challenge \cite{murphy2017disaster}.
To overcome this challenge, the team of robots must decide on a common coordinate frame of reference and localize themselves in that new frame of reference, which will help achieve subsequent multi-robot tasks such as formation control, object manipulation, etc. Most often, UGVs and UAVs collaborate on the task for improved performance and capabilities \cite{kruijff2015tradr}. See Fig.~\ref{fig:overview} for a pictorial representation of this deployment.

\begin{figure}
    \centering
    \includegraphics[width=\linewidth]{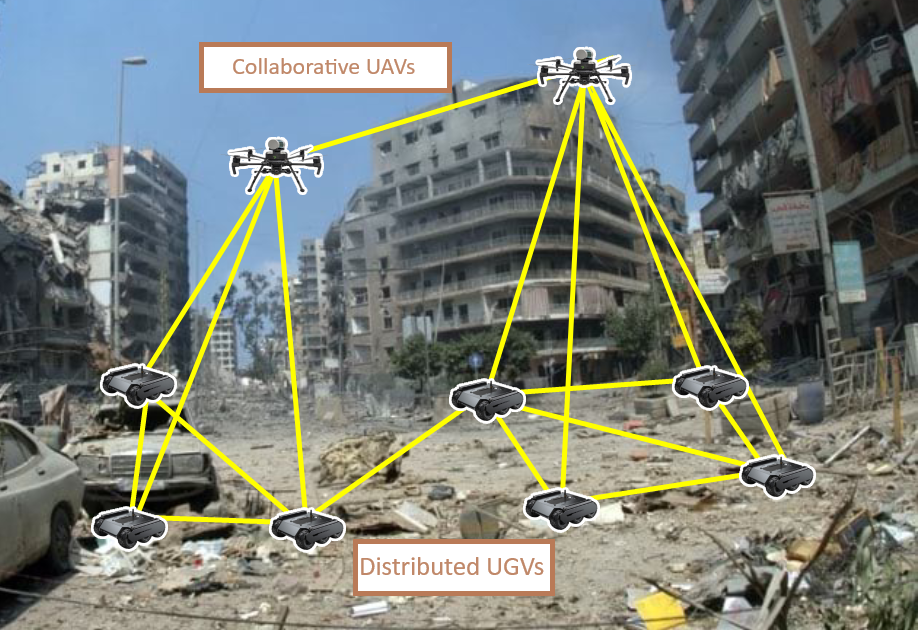}
    \caption{\footnotesize Illustration of the proposed UGV-UAV cooperative localization in a typical disaster scenario. Here, the links shown in "Yellow" are the edges of the sparsely-connected sensing and/or communication graph.}
    \label{fig:overview}
    \vspace{-6mm}
\end{figure}

As noted in \cite{eren2004rigidity,connelly2020generically}, a generic globally rigid (GGR) graph is a sub-type of graph which have been used in works such as \cite{calafiore2012distributed} to ensure the rigidity of the inter-robot connection for applications in distributed approaches. In scenarios where the graph's rigidity is insufficient, the network consensus may not converge but possibly be stuck perpetually in a local minimum \cite{soatti2014weighted,parasuraman2019consensus,tasooji2024event}. 
This is also possible in rigid-connected graphs due to low-edge connectivity. 
These issues are coined as the "flipping phenomenon" \cite{ping2020flipping}, where a vertex can be reflected across a face and still preserve the relative distance constraints (inter-robot range measurements) of the graph. 
In MRS, this scenario can occur if there is limited connectivity to preserve bandwidth or if the formation size (i.e., the configuration of the robot positions) exceeds the maximum range of the robot sensors \cite{kargar2022cooperative}, \cite{tasooji2023decentralized}, \cite{tasooji2024event}. This "flipping" can be caused by the ambiguity in the pose estimate of the neighboring robot, causing the cooperative localization estimates to converge at a local minimum, satisfying the edge constraints but not truly representing the ground truth conditions. Fig.~\ref{fig:flipping} depicts a motivating example.

Another limitation in multi-robot localization arises from the presence of noise in the localization process, which impacts the accuracy and reliability of position estimates. This noise can lead to erroneous distance measurements and pose estimates, ultimately degrading the precision of localization. In practice, localization algorithms may be sensitive to these inaccuracies, potentially causing the robots to converge on incorrect solutions. This can exacerbate issues such as the 'flipping phenomenon,' where, despite a globally rigid (GGR) graph, the system may still converge to a local minimum rather than the true solution.

{\color{black}
\textbf{Contributions}: To overcome the limitations of existing methods in noisy and sparsely connected environments, we introduce a novel range-based cooperative localization framework for robot swarms in GPS-denied settings. Unlike previous approaches, which often rely on leader selection \cite{cao2018dynamic} or edge connectivity manipulations \cite{calafiore2012distributed, engin2021establishing} and struggle with convergence in non-GGR guaranteed graph formations, our method leverages two key innovations: the S1-Edge approach and the integration of UAVs. The S1-Edge technique addresses the rigidity challenges of sparse graphs, improving structural stability, while UAVs, with their superior sensing capabilities and broader field of view, enhance the localization process by accelerating convergence even in environments with weak connectivity. This multi-layered framework combines shadow edge localization with the strategic deployment of UAVs, offering improved scalability and adaptability for large-scale multi-robot systems. Theoretical analysis establishes bounds on the gain coefficient for shadow edges, ensuring the algorithm's stability and accuracy in the presence of noise. Our contribution is further validated through extensive simulation experiments, demonstrating significant improvements in localization accuracy and robustness over the state-of-the-art.}

A video\footnote{\url{https://youtu.be/AhrEJp2i0fw}} demonstration of the simulations and real-world experiments show the approach's strengths and novelty.

\begin{figure}[tbp]
\centerline{\includegraphics[width=\columnwidth]{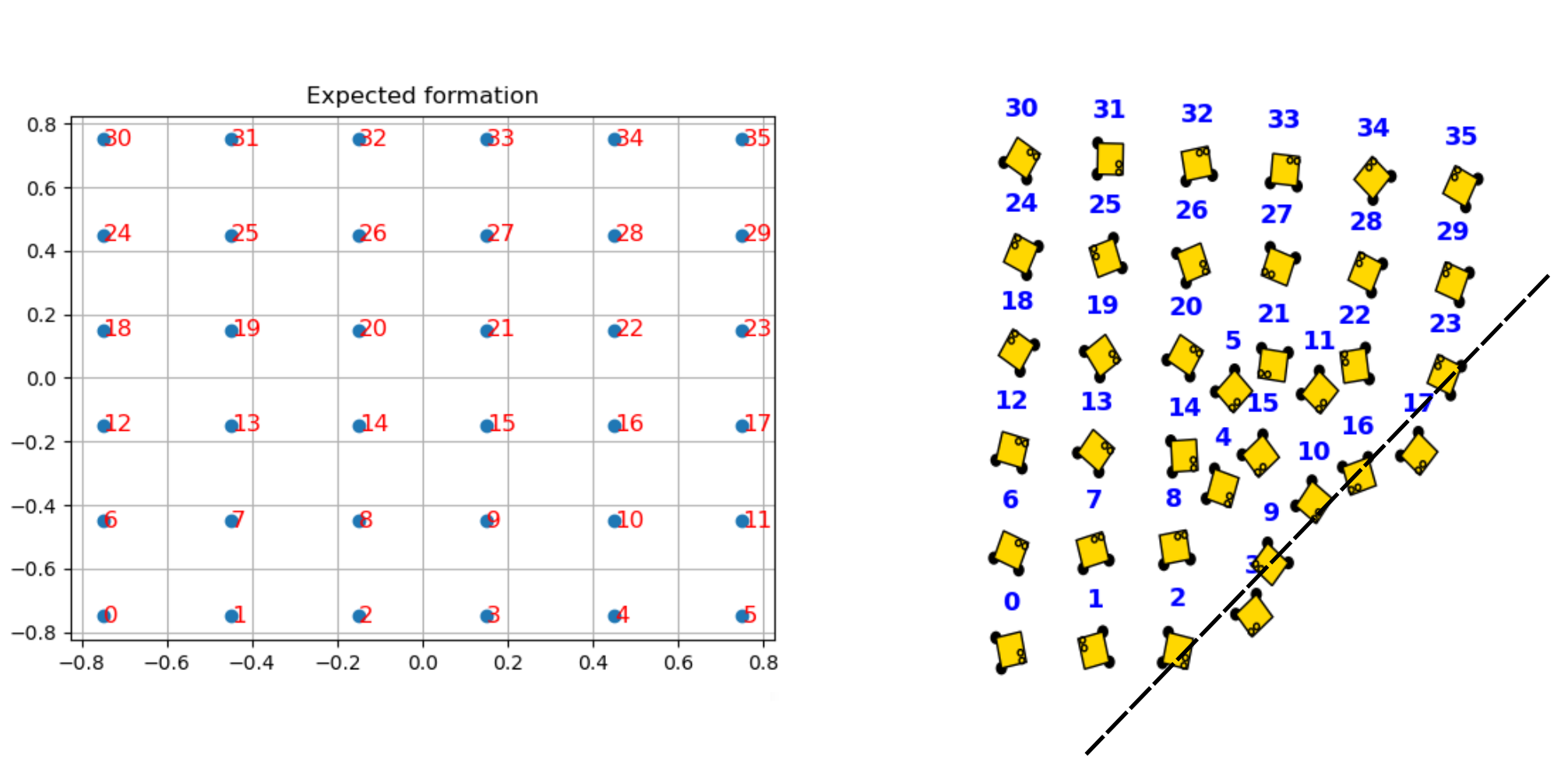}}
\vspace{-4mm}
\caption{\footnotesize Flipping of sparse (non-rigid) sensing graphs with ground truth location (left) and estimated (flipped due to sparsity of range inference) locations (right). Robots 4,6, 10, and 11 are flipped along the axis connecting the robots 2 and 23 in this example.}
\label{fig:flipping}
\vspace{-6mm}
\end{figure}

\section{Related work}
{\color{black} In the literature, cooperative localization has been extensively studied, with various methods proposed for MRS \cite{prorok2012low,nguyen2021flexible,wanasinghe2015relative, chai2014consensus, soatti2014weighted, shen2024theoretical, chakraborty2020cooperative, guo2019localization, de2015ultrasonic}. Wiktor and Rock \cite{wiktor2020ICRA} introduced a Bayesian optimization-based method for collaborative localization in natural terrain, but it involves high complexity and computational costs.

In \cite{latif2022dgorl}, the authors suggested a graph-theoretic approach that addressed some scalability issues but still faced high computational complexity and required extensive and predictive range (or RSSI) sensing, which may be impractical in resource-constrained systems. On the contrary, a range-only sensor could be more practical for MRS operations \cite{quattrini2020multi}. 
Authors in \cite{calafiore2012distributed,latif2023instantaneous} proposed variants of gradient-based approaches for localization in geometrically rigid formations, using stationary beacons for accurate position information. However, these methods require multiple specialized robots (or anchor nodes) and rely on a fully connected graph, making them less practical in large MRS applications. In contrast, our approach only needs a single UAV for a group of UGVs, covering larger areas more efficiently with less bandwidth. It minimizes frame-of-reference reassignment and reduces communication rounds for trilateration.

Cao \emph{et al.} \cite{cao2018dynamic} developed an algorithm where a leader and two reference robots are elected, and trilateration is used for coordinate determination. This method assumes all robots can communicate with reference robots and may need to re-elect a leader if outlier robots cannot communicate. Our distributed controller operates effectively in scenarios with limited connectivity, using anchors and range-only measurements to achieve localization in sparse graphs. 

Oliva \emph{et al.} \cite{oliva2015sensor,oliva2015noisy} proposed a localization method using shadow edges to relax the global rigidity condition in unit disk graphs. They introduced an algorithm with a necessary and sufficient localizability condition, showing that shadow edge localization can succeed where trilateration fails. 
Their method, effective in noisy settings, localizes more nodes and achieves lower localization error compared to trilateration, particularly after recursive least squares post-processing. But, their shadow edges are centralized and not limited to 1-hop neighbors, which limits the distributed application. 
In contrast, our method combines distributed shadow 1-hop edges with UAV deployment for robot swarms in GPS-denied environments, addressing nonrigid graphs and improving convergence and accuracy. Our approach ensures robust performance in noisy conditions and scales effectively in large networks.

Park \emph{et al.} \cite{isprs-annals-V-1-2020-317-2020} proposed a multilateration approach for UAV positioning using Ultrawide-band (UWB) ranging to address flip ambiguity caused by planar anchor configurations. Their method involves an algebraic solution, flip correction through symmetric reflection, and non-linear optimization, collectively enhancing positioning accuracy. Their approach relies heavily on precisely known anchor locations and assumes a relatively low level of noise in the UWB ranging measurements. However, most real-world deployments involve high measurement noise, environmental interference, and dynamic obstacles that can degrade range sensor performance, leading to inaccuracies in positioning.

Moreover, these strategies generally assume ideal geometric anchor arrangements and do not account for irregular or sparse node distributions, which are often encountered in practical UAV navigation scenarios. This limitation makes their solution less robust for dynamic and complex environments where robot placement cannot always conform to a predefined configuration. Furthermore, some of these methods incur non-linear optimization, which suffers from computational overhead, limiting the method's real-time applicability for UAVs operating in time-critical missions. 
Our approach addresses these shortcomings by combining shadow edge localization with UAV deployment in a distributed protocol. This provides a scalable and efficient solution for MRS operating in GPS-denied environments. This method reduces computational and communication demands while enhancing localization accuracy and robustness, representing a notable advancement in practical MRS.

}

\section{Problem Formulation}

In this paper, we are concerned with the graph localization problem with relative range measurements $z$, which can be represented as follows.
Let there be a set of \(n\) nodes (representing the robots) \(V={v_1,...,v_n}\)  in a fixed formation shape (static configuration) and let \(P={p_1,...p_n}\) be a set of positions on the Cartesian plane, where \(x_{i} = p_i = [p_{i,x},p_{i,y}]^T \) represent the 2D coordinates of the \(i^{th}\) node. For an expected, time-invariant formation, the main assumption is that each robot has a range sensor that can provide sensor readings within a limited range; 
Following the unit disk sensing graph, the robots form a connected sensing graph ${G} = (V, E)$ defined by their structural edges $E = \{(i,j) | j \in {N}_i \iff |p_i - p_j| \leq R\}$, where $R$ is the sensing range, ${N}_i$ is the neighbor set of robot $i$ in the graph ${G}$.

\begin{enumerate}
    \item All nodes start with a random position estimate $\hat{x}_i(0)$.
    \item For a node \(i\in \{1, 2,...,n \} \), the neighbors of \(i\) can be measured, and the node adjusts its position estimate based on this information.
    \item In the graph \(G\), only the relative distance measurements are available, and if a node \(i\) can detect its neighbor \(j\), then \(j\) must also be able to detect \(i\). 
    \item There may be special nodes called UAVs in the graph. UAV \(e\in V\), has strong connectivity, i.e., it can be detected by every node in the formation.
    \item If a robot can sense its neighbor, they both can communicate information. i.e., the sensing graph and communication graph have the same edges.
    \item \color{black}{ A network of \( n \) nodes is globally rigid (GGR) if it is defined by a set of pairwise distance measurements \( \{d_{ij}\} \) such that the configuration of the nodes is unique up to global rigid transformations (translations and rotations), and the distance measurements are minimal, meaning that the network has a minimum radius of distance constraints required to ensure this rigidity.}
\end{enumerate}
The localization problem is to determine the $x_i$ for each node in a new cooperative frame of reference.

\section{Approach}
Each robot iteratively updates its position estimates $\hat{x}_i$ at \(\tau\) to establish a common coordinate frame of reference. 

\subsection{Distributed Range-only Localization}
Through local neighborhood communication, robots iteratively update their position estimates by sharing current estimates and range sensor readings with the neighbors using the following update rule \cite{calafiore2012distributed}:
\begin{equation}
 \begin{array}{l}
\hat{x}_i(\tau + 1) = \hat{x}_i(\tau) +  \mathlarger{\mathlarger{\sum}}_{j \in N_i}  \bigg [ \alpha \Lambda_{ij}(\tau) \cdot \Big (\hat{x}_j(\tau) - \hat{x}_i(\tau) \Big )  \bigg ]
\end{array}
\label{eqn:distributed}
\end{equation}

The gain for direct edges is denoted by \( \alpha \).
The term \( \Lambda_{ij} \) represents the inconsistency between the range sensor readings \( z_{ij} \) (measurable for $(i,j) \in E$) and the current position estimates \( \hat{x}_i \) and \( \hat{x}_j \) at time \( \tau \), and is given by:
\begin{equation}
\Lambda_{ij}(\tau) = \|\hat{x}_j(\tau) - \hat{x}_i(\tau)\|^2 - z_{ij}^2(\tau)
\label{eqn:lambda}
\end{equation}
Previous localization approaches, such as the one described in Eq.~\eqref{eqn:distributed}, face significant challenges with incomplete sensing and the absence of direct communication between robots. These limitations can result in inaccurate or unstable localization estimates due to the reliance on direct measurements and the assumption of a fully connected network. Such assumptions are often impractical in real-world scenarios, where sparse sensing results in "flipping" problems.

\subsection{Distributed Localization With Shadow 1-Hop Edge}

We address the limitations of the previous localization approach outlined in Eq.~\eqref{eqn:distributed} and propose a new algorithm based on the concept of distributed shadow edges for scenarios with sparse direct sensing. 
A significant challenge in this process is the issue of sparsity. Without a global reference frame (GPS), the robots need to establish a reliable communication network to ensure that they can accurately localize themselves and one another. Connectivity issues can complicate the process, as some robots may be out of range or unable to communicate with others \cite{pandey2022empirical}, leading to incomplete or unreliable distance information. To address this challenge, we introduce the concept of shadow edge localization. In this context, certain edges—like the one depicted by the blue dotted line in Fig. \ref{fig:s1-bounds}, are referred to as S1-Edges. These S1-Edges help address sparsity issues by providing additional reference points, aiding the robots in defining their local coordinate frame and achieving more accurate localization despite the challenges of non-ggr sensing graphs. 

\begin{figure}
    \centering
    \includegraphics[width=0.9\linewidth]{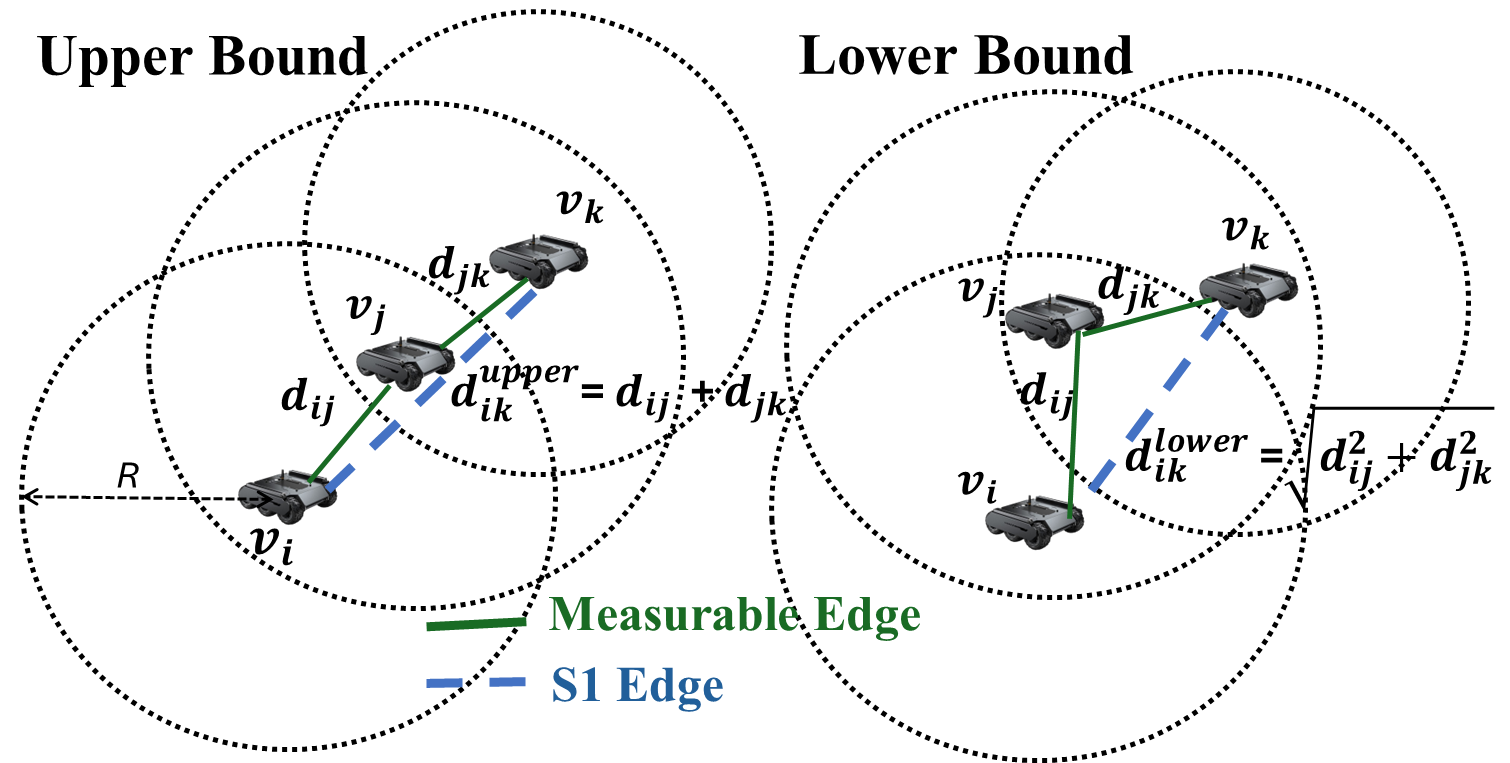}
    \caption{\footnotesize Localization of robots in a network with a unit disk graph where the associated graph G is not ggr. Robots \( r_i \), \( r_j \), and \( r_k \), represented by nodes \( v_i \), \( v_j \), and \( v_k \). Robot \( r_i \) can detect robot \( r_j \), but cannot detect robot \( r_k \) (i.e., node \( v_k \) is outside the dashed circle). The blue dotted line indicates a potential shadow edge that might be formed between \( v_i \) and \( v_k \). Based on the distances \( d_{ij} \) and \( d_{jk} \), upper and lower bounds for \( d_{ik} \) are defined.}
    \label{fig:s1-bounds}
    \vspace{-4mm}
\end{figure}
\par

\textbf{Definition 1: Shadow 1-hop edge (S1-Edge):} 
In graph $G$, a \textit{Shadow 1-Hop Edge} exists between a pair \( (i, k) \notin E \) when robot \( k \) is indirectly connected to robot \( i \) via an intermediate node \( j \in N_i \), where \( N_i \) is the set of 1-hop neighbors of \( i \).
For robots \( i \), \( j \), and \( k \), where \( j \in N_i \) and \( k \in N_j \), the distance \( d_{ik} \) is bounded by $d_{ik}^{\text{upper}} \leq d_{ik} \leq d_{ik}^{\text{upper}}$:
\begin{eqnarray}
\begin{split}
\textbf{Upper Bound:} & d_{ik}^{\text{upper}} & = d_{ij} + d_{jk} \\
\textbf{Lower Bound:} & d_{ik}^{\text{lower}} & = \sqrt{d_{ij}^2 + d_{jk}^2}, 
\end{split}
\label{eqn:bounds}
\end{eqnarray}
where \( d_{ij} \) and \( d_{jk} \) are the Euclidean distances between the consecutive robots that share an edge $E$.
The Upper bound ($d_{ij}^{\text{upper}}$) is derived from the co-linear property, while the lower bound ($d_{ij}^{\text{lower}}$) is derived from the triangle inequality, as depicted in Fig.~\ref{fig:s1-bounds}:

The S1-Edge concept enables flexible distance estimation between non-adjacent robots via intermediate nodes. The upper and lower bounds provide key insights into accuracy and feasibility. These bounds ensure reliable estimation when direct measurements are unavailable, with the upper bound covering worst-case scenarios and the lower bound reflecting the shortest geometric distance. This approach guarantees convergence and stability in MRS, especially where maintaining communication is critical. Limiting shadow edges to 1-hop neighbors simplifies implementation and ensures scalability in distributed frameworks.

\noindent \textbf{Distributed Protocol} Our algorithm is defined by the following update rule, which allows a distributed utilization of the S1-Edges integrated in the distributed protocol of \eqref{eqn:distributed}.

\begin{eqnarray}
\hat{x}_i(\tau + 1) = \hat{x}_i(\tau) +  \mathlarger{\mathlarger{\sum}}_{j \in N_i}  \bigg [ \alpha \Lambda_{ij}(\tau) \cdot \Big (\hat{x}_j(\tau) - \hat{x}_i(\tau) \Big )   \nonumber \\  \ \ \ +  \mathlarger{\mathlarger{\sum}}_{k \in N_j/N_i } \beta S_{ik}(\tau) \Lambda_{ik}(\tau) \Big (\hat{x}_k(\tau) - \hat{x}_i(\tau) \Big ) \bigg ] ,
\label{eqn:n+} 
\end{eqnarray}
The set of shadow S1-Edges represented as \( k \in N_j / N_i \), representing shadow edges that are not in the direct neighborhood set \( N_i \) of robot \( i \) and  \( S_{ik} \) represents the virtual interaction between robot \( i \) and robot \( k \) (i.e., $k \notin N_i \textit{ if } |x_i - x_k| > R$). Note that \( S_{ik}(\tau) =1 \) if the robot \( i \) and robot \( k \) are within the range $R$ of each other as per the current estimate $\hat{x}(\tau)$ (but they are not in reality); otherwise, it is set to zero. 
The gain for shadow edges is denoted by \( \beta \).

Note that $\Lambda_{ik}(\tau)$ is not directly available as $z_{ik}$ is not measurable per the sensing graph $G$ since $(i,k) \notin E$. However, we leverage the knowledge of the S1-Edge bounds we established earlier to obtain an estimate $\hat{z}_{ik}$ of the virtual edge $(i,k)$ by averaging the upper and lower bounds of $d_{ik}$ \eqref{eqn:bounds} using the range measurements of $(i,j)$ and $(j,k) \in E$, which are available or sharable in 1-hop (distributed) manner per $G$.
\begin{equation}
    \hat{z}_{ik} = \frac{1}{2} [(z_{ij} + z_{jk}) + (\sqrt{z^2_{ij} + z^2_{jk}})]
\end{equation}
Accordingly, we use $\Lambda_{ik}(\tau) = \|\hat{x}_j(\tau) - \hat{x}_k(\tau)\|^2 - \hat{z}_{ik}^2(\tau)$.
Here, the information about $\hat{x}_k(\tau)$ is communicated to the node $i$ by its 1-hop neighbor node $j$ to realize the algorithm as a distributed implementation.

The S1-Edge term \(  \sum_{k \in N_j/N_i} \beta S_{ik} \Lambda_{ik} (\hat{x}_k - \hat{x}_i) \) is essential for improving localization performance by leveraging additional virtual connections, thus enhancing estimate accuracy in sparse and poorly connected graphs. 
In other words, the S1-Edge term significantly compensates for the lack of direct measurements by incorporating indirect connections, thereby improving localization accuracy. This iterative process allows robots to achieve consensus, aligning their estimates with range measurements and establishing a common reference frame for relative positioning.

\subsection{UAV Nodes to Speed up the Convergence}
\label{sec:emitter}
To speed up the convergence process further, UAVs (powerful sensing nodes, aka emitters) are introduced to the network. They provide extended-range detection and overcome the communication limits of edge neighbors. 
Mathematically, we assume there is a UAV robot \(e\in V\), which can sense the range of all robots (or a subset of robots in a distributed setting). i.e., $ (i,e) \in E \forall i \in V$.
Placed strategically in the robot formation, The reflection of some nodes as emitters modifies the adjacency matrix, thereby enhancing connectivity and convergence rates. Experimental data supports this approach, demonstrating that powerful sensing nodes help apply in systems with noisy measurements and facilitate faster convergence to a global minimum. We use one UAV as an emitter in this paper and propose a variant "Emitter-based localization", which applies the distributed localization in \eqref{eqn:distributed} without considering the S1-Edges. 

\subsection{DCL-Sparse:}
The proposed DCL-Spare algorithm combines the advantages of S1-Edge and the UAV node together, and the whole process is summarized in the pseudocode of Alg.~\ref{alg:dcl-sparse}.

\vspace{-4mm}
\begin{algorithm}
\caption{Distributed DCL-Sparse Localization}
\label{alg:dcl-sparse}
\footnotesize 
\begin{algorithmic}
    \State \textbf{Initialize:}
    \State $V \gets \{v_1, v_2, \ldots, v_n\} \gets$ {Nodes (robots)}
    \State $X = (\hat{x}_1,\hat{x}_2,\ldots,\hat{x}_n) \gets$ Initial position estimates (random) 
    \State $G \gets (V, E) \gets$ {Graph based on sensing range}
    \State $\alpha, \beta \gets$ Gain factors
    \State \textbf{Main Loop:}
    
    \For{$\tau = 0$ \textbf{to} max\_iterations}{
        \State \textit{{  }Distributed protocol applied at each robot $i$} 
        
        \For{each $i \in V$}{ 
            \State \textit{Information Propagation of Self + 1-Hop Estimates}
            
            \For{$j \in N_i$  }{ 
                \State \textbf{Send:} $\hat{x}_i$, $(z_{ij},\hat{x}_j) \, \forall j \in N_i$ to $j$ 
                \State \textbf{Receive:} $\hat{x}_j$, $(z_{jk},\hat{x}_k) \, \forall k \in N_j$ from $j$ 
            }
            \State \textit{Include UAV node in the neighbor set}
            
            \For{$j \in N_i \cup \text{UAV}$}{ 
                \State $\Lambda_{ij} \gets \|\hat{x}_j - \hat{x}_i\|^2 - z_{ij}^2$
                \State $\hat{x}_i \gets \hat{x}_i + \alpha \Lambda_{ij} (\hat{x}_j - \hat{x}_i)$
                
                \For{$k \in N_j \setminus N_i$}{
                    \If{$\|\hat{x}_i - \hat{x}_k\| < R$ and $\exists j \in N_i \cap N_k$}{
                        \State $\Lambda_{ik} \gets \|\hat{x}_k - \hat{x}_i\|^2 - \hat{z}_{ik}^2$
                        \State $\hat{x}_i \gets \hat{x}_i + \beta S_{ik} \Lambda_{ik} (\hat{x}_k - \hat{x}_i)$ 
                    }
                }
            }
        }
    }
\end{algorithmic}
\end{algorithm}
\vspace{-6mm}

\subsection{Theoretical Analysis}
In our DCL-Sparse approach, we combine the S1-Edge approach and the Emitter-based approach. This combination leverages the advantages ofof both methods and helps improve convergence and localization accuracy in sparse and noisy sensing graphs. We present the main results.

\begin{theorem}
    Consider a unit disk graph \( G = (V, E) \), where \( V = \{r_1, r_2, \dots, r_n\} \) represents robots in a network and \( E \) represents relative distance measurements based on a threshold \( \rho \). Let \( E_s \) be the set of S1-Edges, representing virtual connections between robots without direct relative distance measurements. The gain coefficient \( \beta \) associated with each S1-Edge must satisfy the following bounds to ensure stable and accurate localization:
\begin{equation}
    \frac{1}{\max(\eta_{ik}, \epsilon)} \leq \beta \leq \frac{1}{\left(\frac{d_{ij} + d_{jk}}{2}\right) + \eta_{ik}} ,
\end{equation}
where:
    \begin{equation}
    \begin{array}{l}
\eta_{ik} = 2 d_{kj} d_{ij} \cos\left( a_{hkj} - a_{hij} \right) + \eta_{kj} + \eta_{ij} \\ \ \ \ \  - 2 \sqrt{(d_{kj}^2 + \eta_{kj})(d_{ij}^2 + \eta_{ij})} \cos\left( \hat{a}_{hkj} - \hat{a}_{hij} \right)
\end{array}
\end{equation}
\begin{equation}
    \begin{array}{l}
\hat{a}_{abc} = \cos^{-1} \left( \frac{d_{ab}^2 + \eta_{ab} + d_{ac}^2 + \eta_{ac} - d_{bc}^2 - \eta_{bc}}{2 \sqrt{(d_{ab}^2 + \eta_{ab})(d_{ac}^2 + \eta_{ac})}} \right)
\end{array}
\end{equation}
    \begin{itemize}
        \item \( \eta_{ik} \) is the noise associated with the shadow edge between robots \( r_i \) and \( r_k \),
            \item \( \eta_{kj} \) and \( \eta_{ij} \) are the noise terms associated with the distance measurements from \( r_k \) and \( r_i \) to \( r_j \), respectively.

        \item \( d_{kj} \) and \( d_{ij} \) are the relative distances from \( r_k \) and \( r_i \) to a common neighboring robot \( r_j \),
        \item \( \epsilon \) is a small positive constant to avoid division by zero when \( \eta_{ik} \) is extremely small,
         \item \( a_{hkj} \) and \( a_{hij} \) are the angles between the lines connecting (\( r_h \),\( r_k \)) and (\( r_h \),\( r_i \)), and the line connecting (\( r_h \),\( r_j \)),
        \item \( \hat{a}_{hkj} \) and \( \hat{a}_{hij} \) are the estimated angles between the lines connecting (\( r_h \),\( r_k \)) and (\( r_h \),\( r_i \)), respectively.
\end{itemize}
These bounds ensure that the shadow edge contributions to the localization process are neither negligible nor disproportionately large, leading to stable and reliable convergence.
\label{thm:bounds}
\end{theorem}
\begin{proof}
    The localization process is governed by the update rule \eqref{eqn:n+}:
    Define the error vector as:
    \begin{equation}
    e(\tau) = \hat{X}(\tau) - X^*
    \end{equation}
    The error dynamics are $e(\tau+1) = (I + \alpha L_d + \beta L_s) e(\tau)$.
    
    For stability, the spectral radius must satisfy \( \rho(I + \alpha L_d + \beta L_s) < 1 \).
    \begin{proofpart}
        \textbf{Lower Bound on \( \beta \)}: 
        $\beta \geq \frac{1}{\max(\eta_{ik}, \epsilon)}$.
        
        This lower bound prevents the S1-Edges from being ignored due to large noise, ensuring they always contribute to the update rule in \eqref{eqn:n+}.
    \end{proofpart}
     \begin{proofpart}
        \textbf{Upper Bound on \( \beta \)}:
        $\beta \leq \frac{1}{\left(\frac{d_{ij} + d_{jk}}{2}\right) + \eta_{ik}}$.
        
        This bound limits the influence of S1-Edges, preventing instability when relative distances are large or noise is low.
    \end{proofpart}
    \begin{proofpart}
        \textbf{Convergence and Stability Analysis}:
        The spectral radius of \( I + \alpha L_d + \beta L_s \) must be less than 1 for stability. The eigenvalues \( \lambda_i \) of \( \alpha L_d + \beta L_s \) must satisfy $-2 < \lambda_i < 0 $. 
        Ensuring \( \beta \) is within the bounds guarantees stability and convergence of the error \( e(\tau) \) to 0, leading to an accurate cooperative localization system in MRS.
    \end{proofpart}
    \vspace{-4mm}
\end{proof}
\begin{remark}
    The proposed localization framework ensures robustness and stability by carefully bounding the gain coefficient \( \beta \) associated with shadow edges. The lower bound on \( \beta \) prevents underutilization of shadow edge information, maintaining efficiency even in the presence of noise. Conversely, the upper bound prevents shadow edges from dominating the localization process, ensuring that direct relative distances retain significance and the system is not destabilized by excessive shadow edge corrections. The stability of the algorithm is guaranteed by ensuring that the eigenvalues of the system matrix \( I + \alpha L_d + \beta L_s \) remain within the range \((-2, 0)\), which supports convergence to true robot positions despite noisy conditions. Noise perturbations, both in direct distance measurements and shadow edges, are accounted for, and the spectral radius condition \( \rho(I + \alpha L_d + \beta L_s) < 1 \) ensures that deterministic error components decay over time. As long as the noise remains bounded, the algorithm maintains stable convergence, with the final error confined within noise-defined limits. This approach effectively mitigates the impact of noise, ensuring stable and convergent localization.
\end{remark}
\begin{remark}
    The introduction of a UAV/Emitter (Sec.~\ref{sec:emitter}) significantly alters the graph Laplacian, thereby enhancing sensing density. This alteration has a direct impact on both the estimation error and the convergence rate of the localization. Specifically, improved connectivity through the emitter accelerates convergence and reduces estimation error, leading to more accurate and efficient localization performance.
\end{remark}


\begin{figure*}[t]
    \centering
    \includegraphics[width=0.45\linewidth]{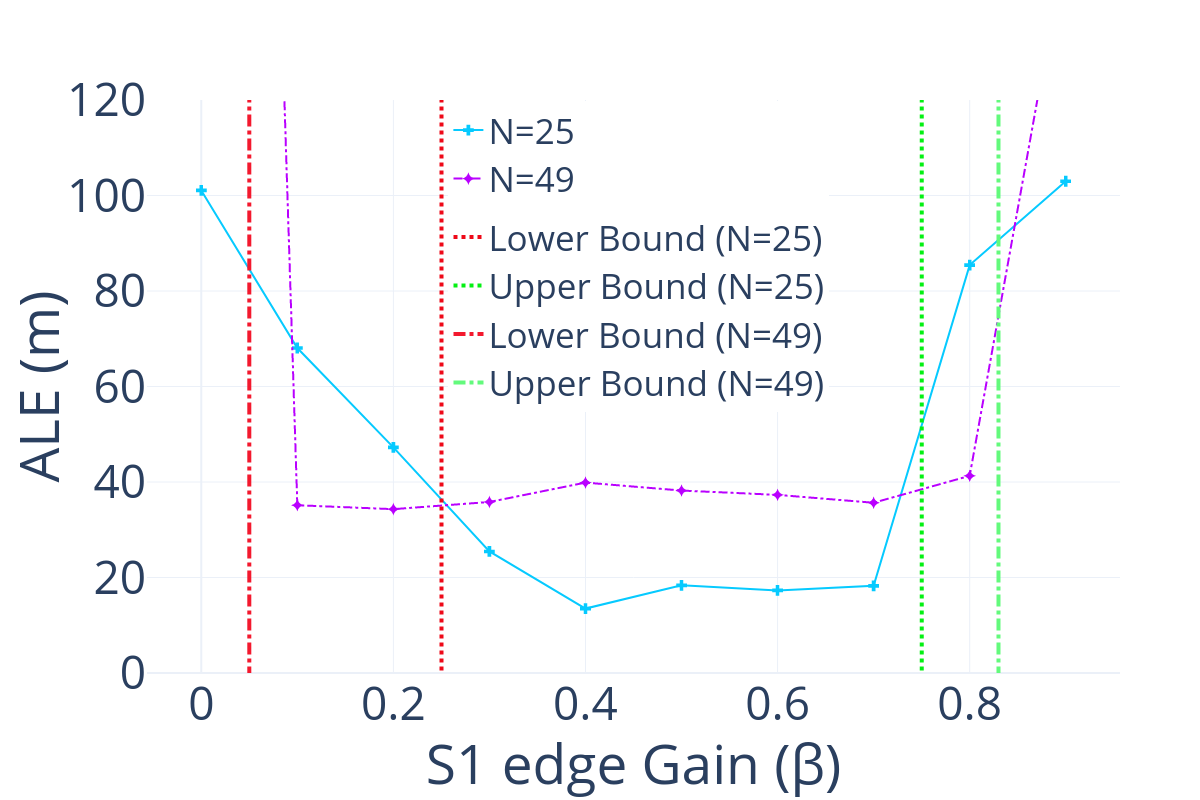}
    \includegraphics[width=0.45\linewidth]{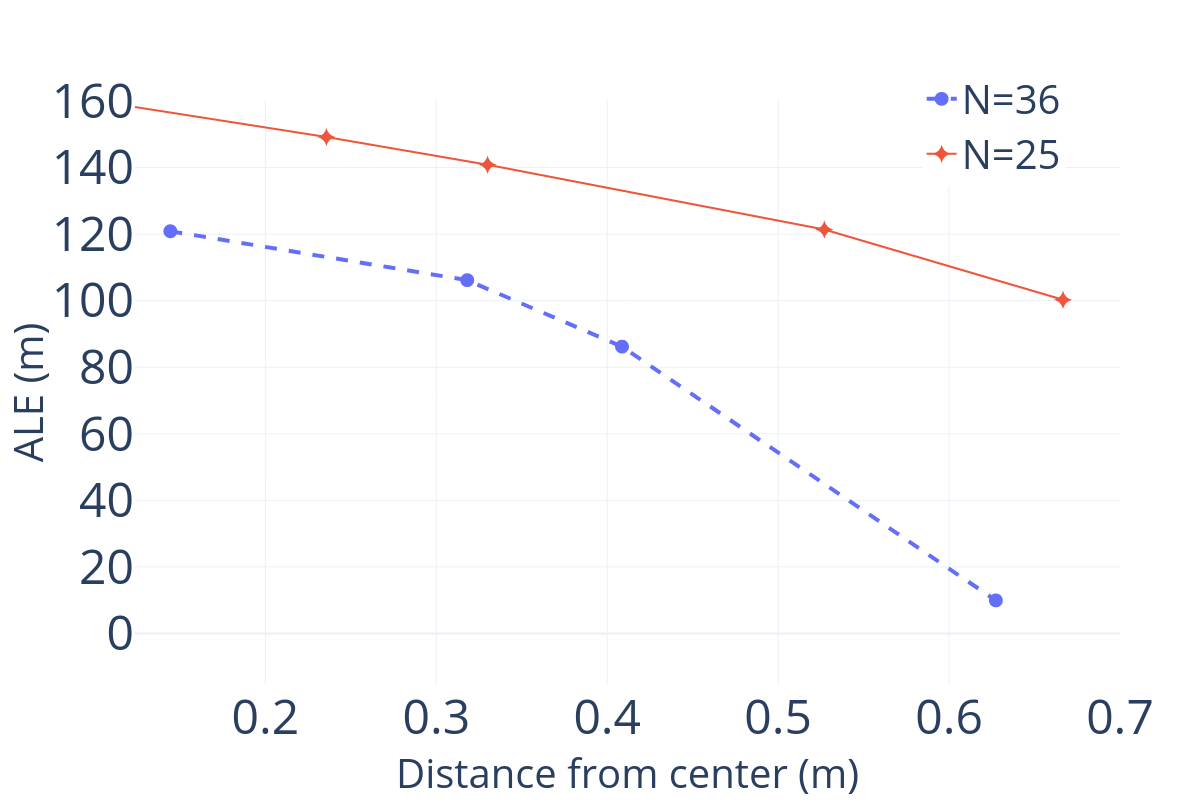}
     \caption{ (a) Confirming the bounds of S1-Edge gain (i.e., $\beta$ in \eqref{eqn:n+}) in the DCL-sparse approach. (b)  Determining optimal placement for the Power node among the robots.}
     \label{fig: Gain}
 \end{figure*}

\begin{figure*}[t]
    \centering
    \setlength{\tabcolsep}{1pt} 
    \resizebox{\textwidth}{!}{
        \begin{tabular}{ccccc}
            \includegraphics[width=0.18\linewidth]{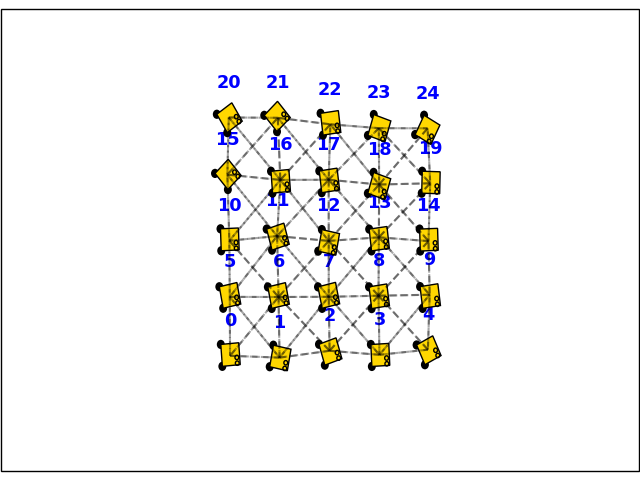} &
            \includegraphics[width=0.18\linewidth]{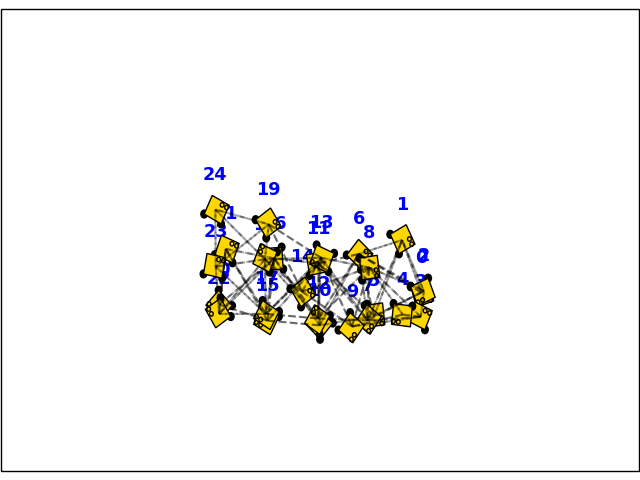} &
            \includegraphics[width=0.18\linewidth]{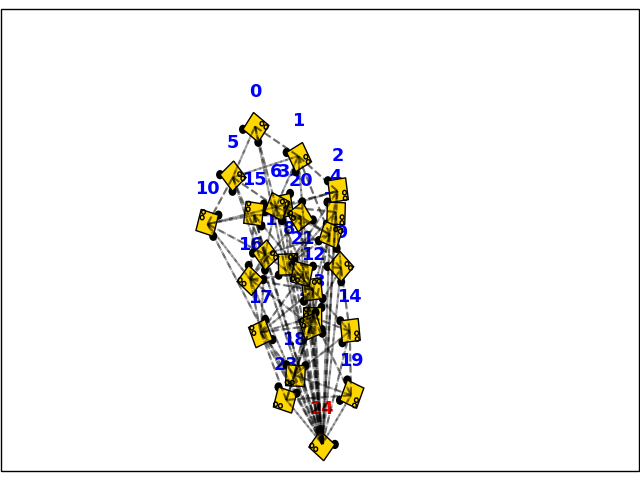} &
            \includegraphics[width=0.18\linewidth]{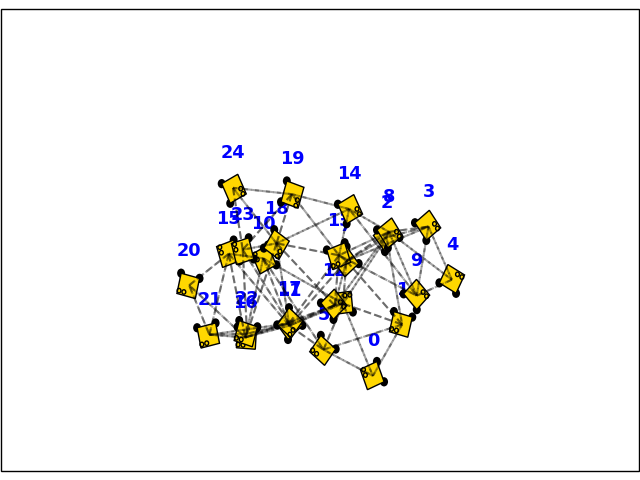} &
            \includegraphics[width=0.18\linewidth]{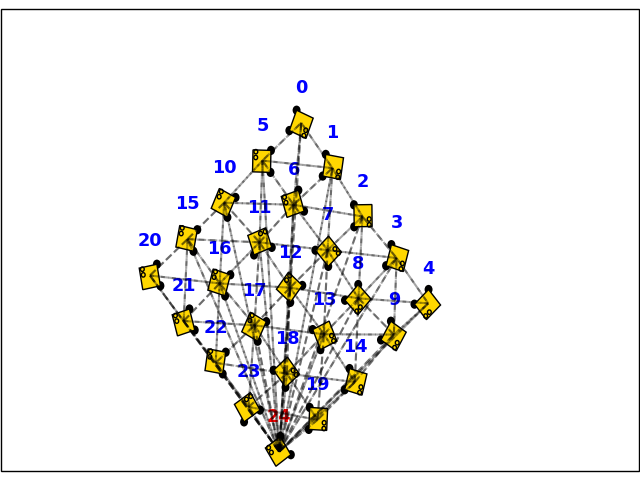} \\
            \footnotesize (a) Ground Truth &  \footnotesize (b) Baseline Localization & \footnotesize (c) Emitter-based Localization &  \footnotesize (d) S-1 Edge Localization & \footnotesize (e) DCL-sparse Localization
        \end{tabular}
    }
    \caption{Final localization results for distributed cooperative localization in $\mathbb{R}^2$. Given that our approach is relative localization, the orientation of the achieved formation shape of the estimated positions should be ignored.}
    \label{fig: formation}
\end{figure*}

\begin{figure*}[t]
    \centering
    \setlength{\tabcolsep}{1pt} 
    \resizebox{\textwidth}{!}{
        \begin{tabular}{ccc}
            \includegraphics[width=0.22\linewidth]{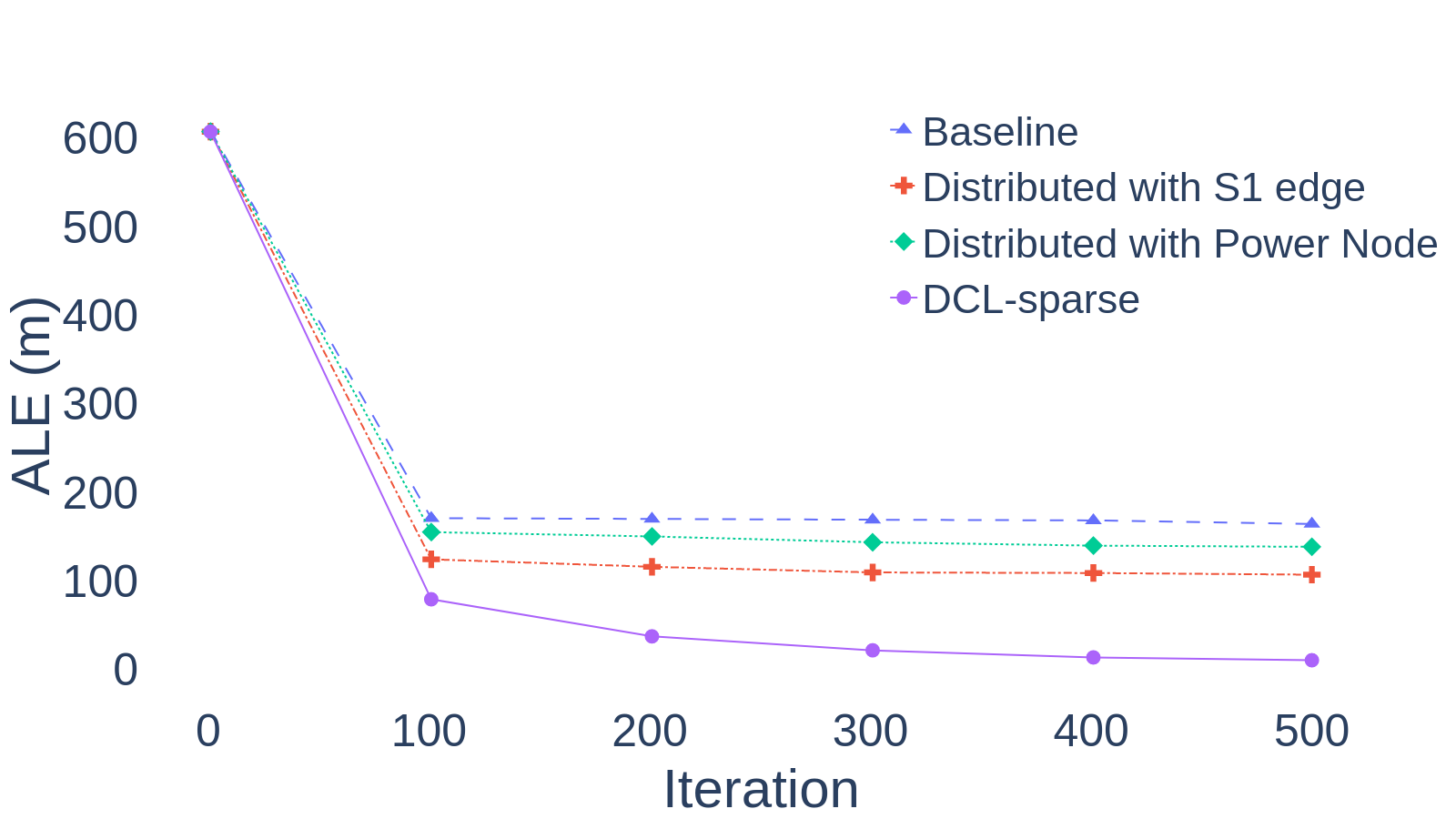} &
            \includegraphics[width=0.22\linewidth]{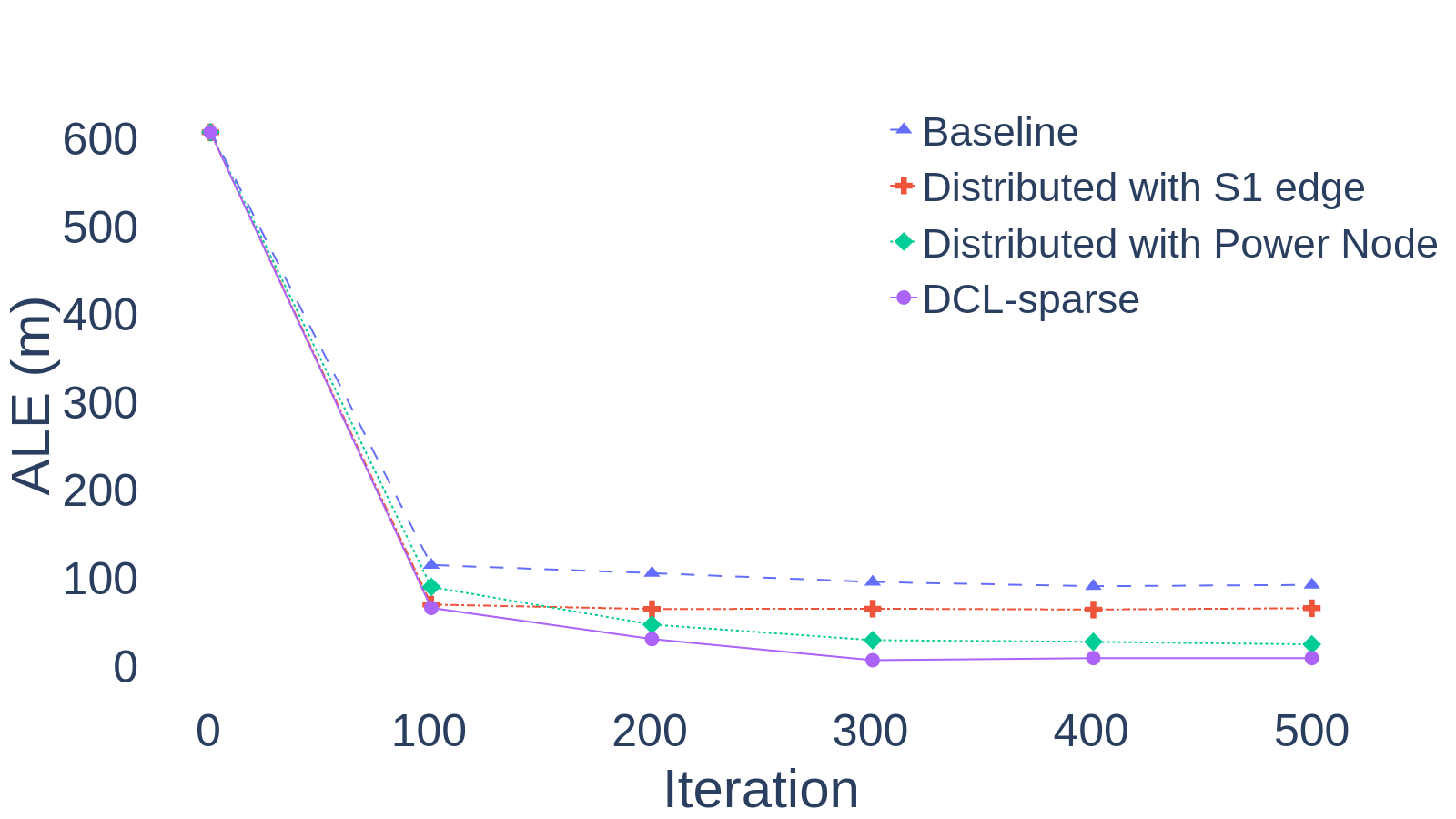} &
            \includegraphics[width=0.22\linewidth]{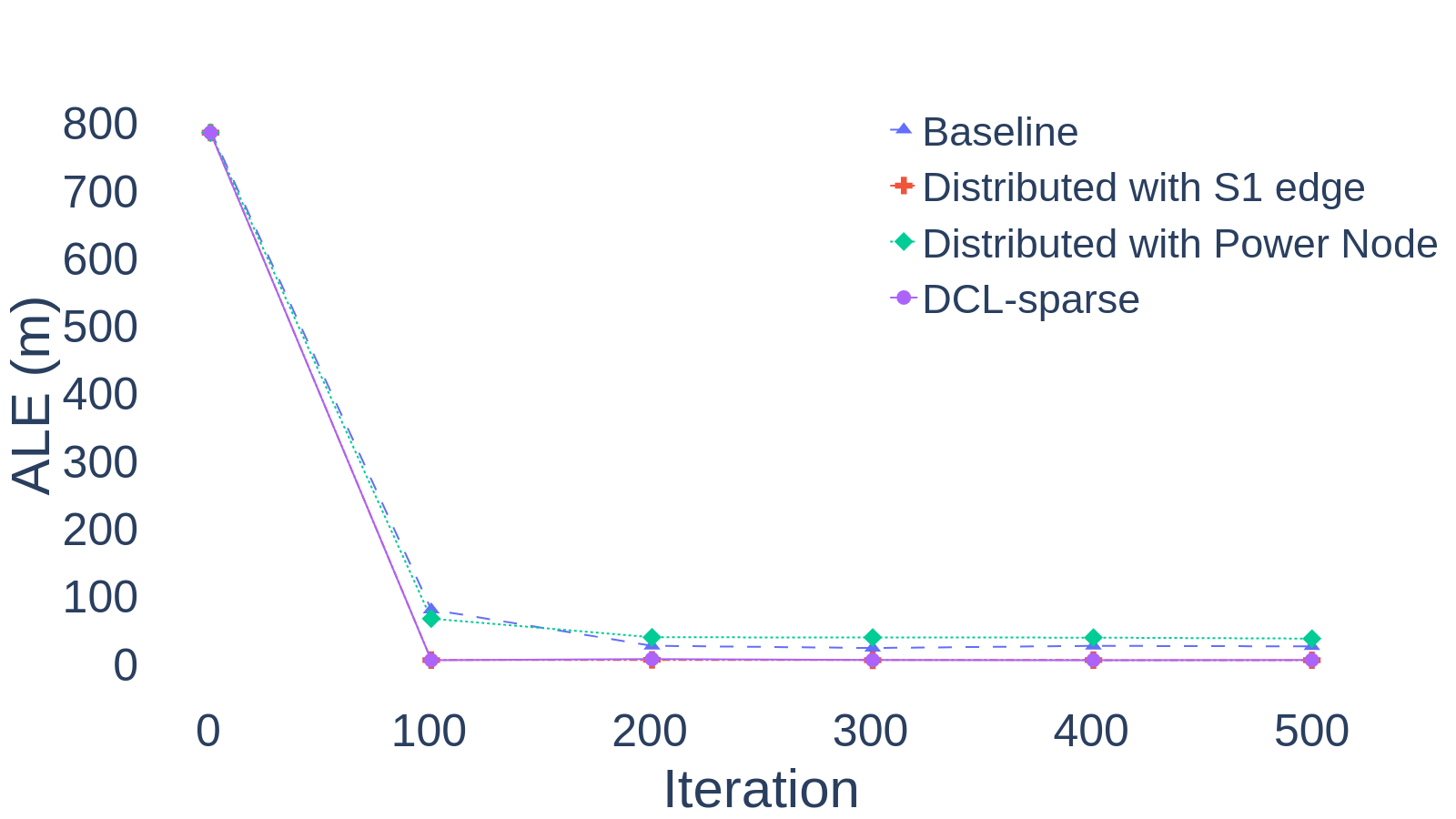} \\
            {\tiny (a) Radius 0.4 (non-ggr graph)} & \tiny (b) Radius 0.67 (at ggr threshold) & \tiny (c) Radius 1.5 (ggr graph) 
        \end{tabular}
    }
    \vspace{-2mm}
    \caption{Localization results with different graph sparsity. }
    \label{fig: sparsity}
    \vspace{-4mm}
\end{figure*}

\section{Experiments}
\noindent \textbf{Experiment Setup:}
In our experimental setup, we set up 25 robots in the Robotarium \cite{wilson2019robotarium} simulation-hardware testbed, which provides remote accessibility to the hardware multi-robot testbed.
By default, the robots are arranged in a square lattice, bound within the coordinates of the 3mx3m testbed. 
Due to the leaderless approach, we use the algorithm in \cite{calafiore2012distributed} as the baseline distributed localization. We also compare variants of our integrated DCL-Sparse approach without the influence of the Emitters (i.e., S1-Edge Localization) and the S1-Edges (i.e., Emitter-based Localization). 
In this case, we consider the emitter, and we test the different placements for the UAV. 
Initially, we set the parameters \( \eta_{kj} \) and \( \eta_{ij} \) to 0.05, the radius of the robot's range to 0.4, and the number of robots to 25 for our experiments. We define the accumulation of absolute localization error (ALE) as 
\begin{equation}
    ALE = \sum_{i=1}^{N} \sum_{\substack{j=1 \\ j \neq i}}^{N} | \|(\mathbf{x}_i - \mathbf{x}_j\| - \|(\hat{\mathbf{x}}_i - \hat{\mathbf{x}}_j)\| | ,
\end{equation}
where $\mathbf{x}$ is the position of robots expressed in the ground truth reference frame and $\hat{\mathbf{x}}$ is the outcome of the distributed localization. 

As proven in \cite{eren2004rigidity}, for a unit disk graph with a sensing radius \(R\) in a unit square area, if \(R>2\sqrt{2}\sqrt{\log(n)/n}\), with \(n\) being the number of robots, then there is a very high probability that the graph is GGR. 
We use this equation to set the sensing radius appropriately to simulate different sparsity conditions in the sensing graphs. 

\subsection{Impact of S1-Edge Gain on Localization Performance}
We examine how the ALE relates to $\beta$ in \eqref{eqn:n+} of the DCL-sparse with \( N=25 \) and \( N=49 \) robots. According to Theorem~\ref{thm:bounds}, increasing $n$ enhances network density and connectivity but also introduces more noise (signal interference), which can decrease the lower bound of ALE. Additionally, more robots generally lead to closer distances between them, resulting in a higher upper bound for ALE and a greater influence of shadow edges. Despite these factors, the ALE for the DCL-sparse algorithm remains low when the S1-Edge Gain is appropriately set within the bounds, as shown in Fig.~\ref{fig: Gain}(a), thereby validating our theoretical analysis. 
Accordingly, we set $\beta=0.5$ in further experiments.

\subsection{Optimal UAV Placement for Enhanced Network Coverage}
We determine the optimal placement for the UAV/Emitter node. We used N=25 and N=36 robots, and $R$ is set to be at the GGR threshold. Fig.~\ref{fig: Gain}(b) shows that the ALE decreases as the emitter node is placed further from the center of the robots' cluster. This improvement is attributed to enhanced sensing coverage, as moving the emitter (UAV) outward allows it to cover a larger portion of the network, extending its communication and sensing range. The wider coverage improves the density of the sensing graph, facilitating better information exchange and localization. Accordingly, we set one of the boundary nodes as UAV in further experiments.

\begin{figure*}[t]
    \centering
    \setlength{\tabcolsep}{1pt} 
    \resizebox{\linewidth}{!}{
        \begin{tabular}{cc}
            \includegraphics[width=0.49\linewidth]{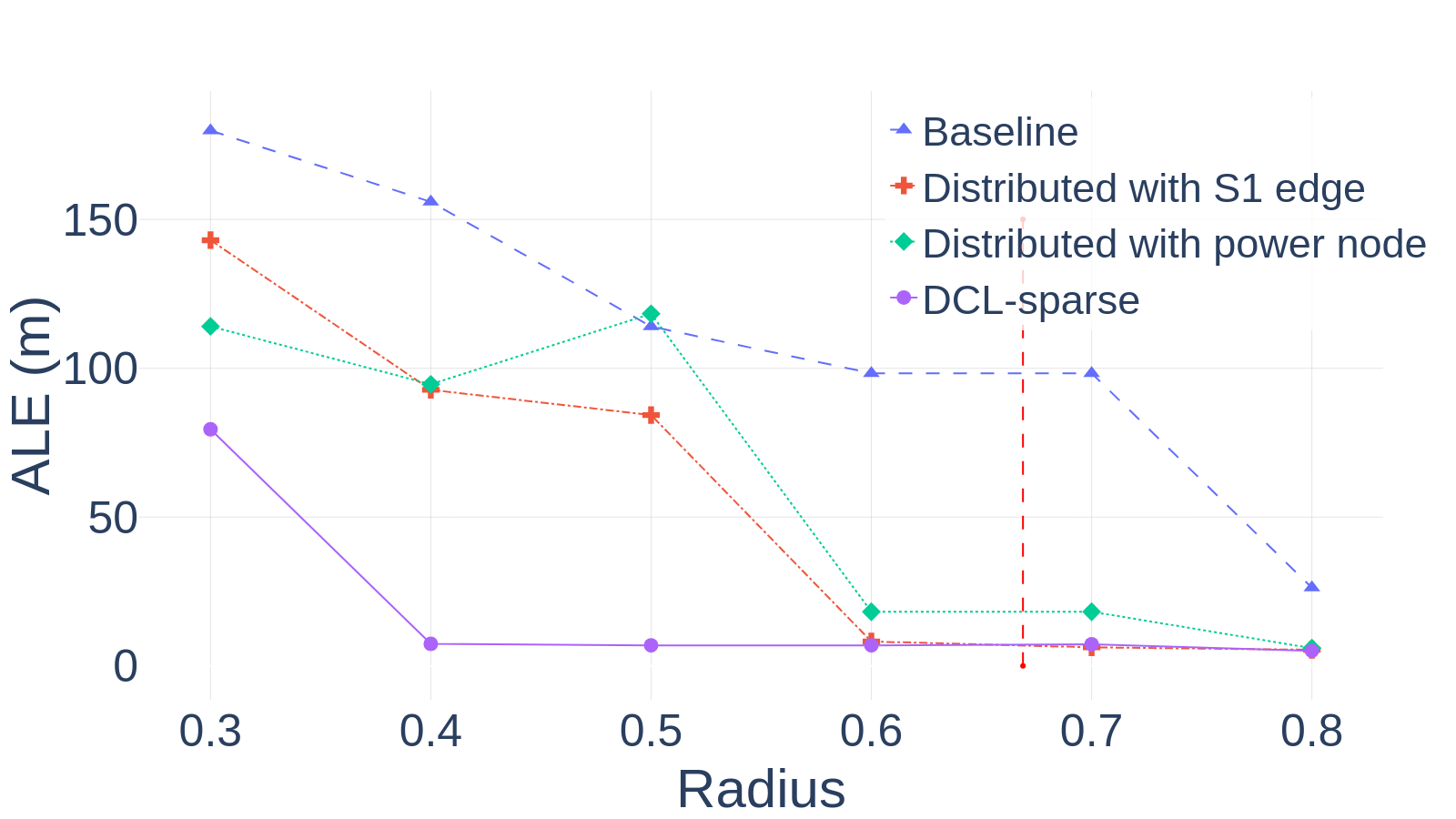} &
            \includegraphics[width=0.49\linewidth]{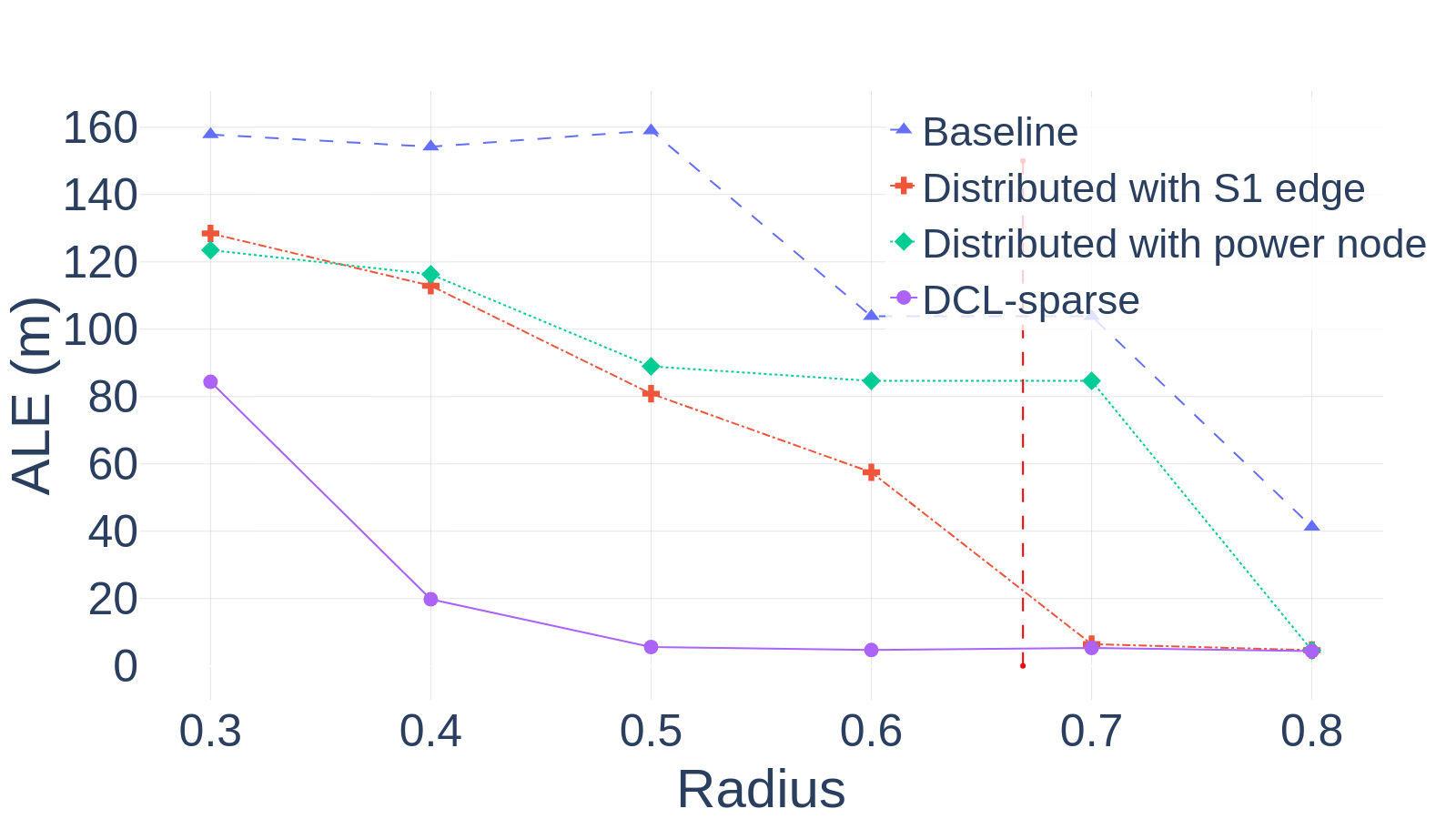} \\
            \footnotesize (a) Square shape formation & \footnotesize (b) Random shape formation \\
            \includegraphics[width=0.49\linewidth]{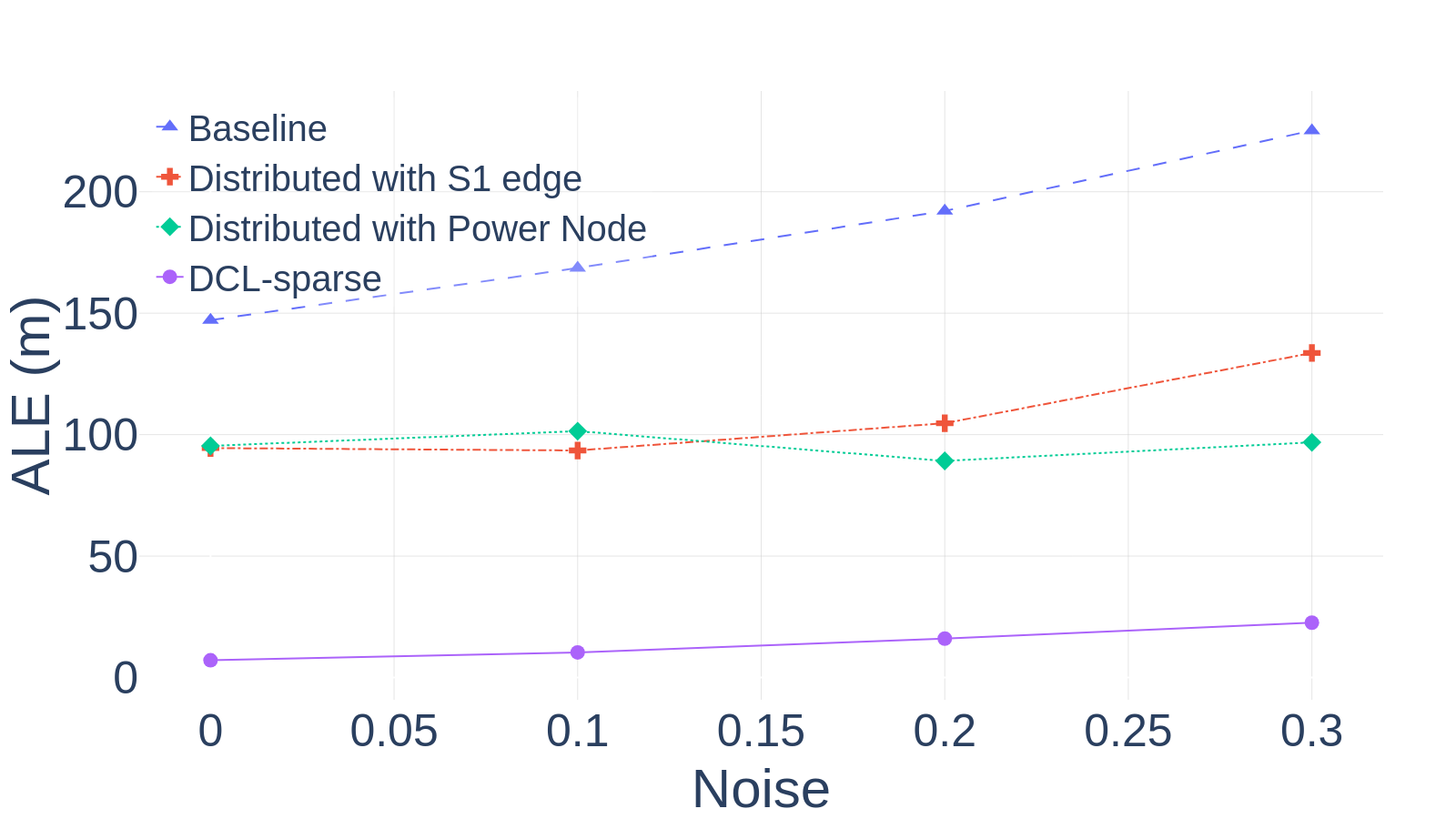} &
            \includegraphics[width=0.49\linewidth]{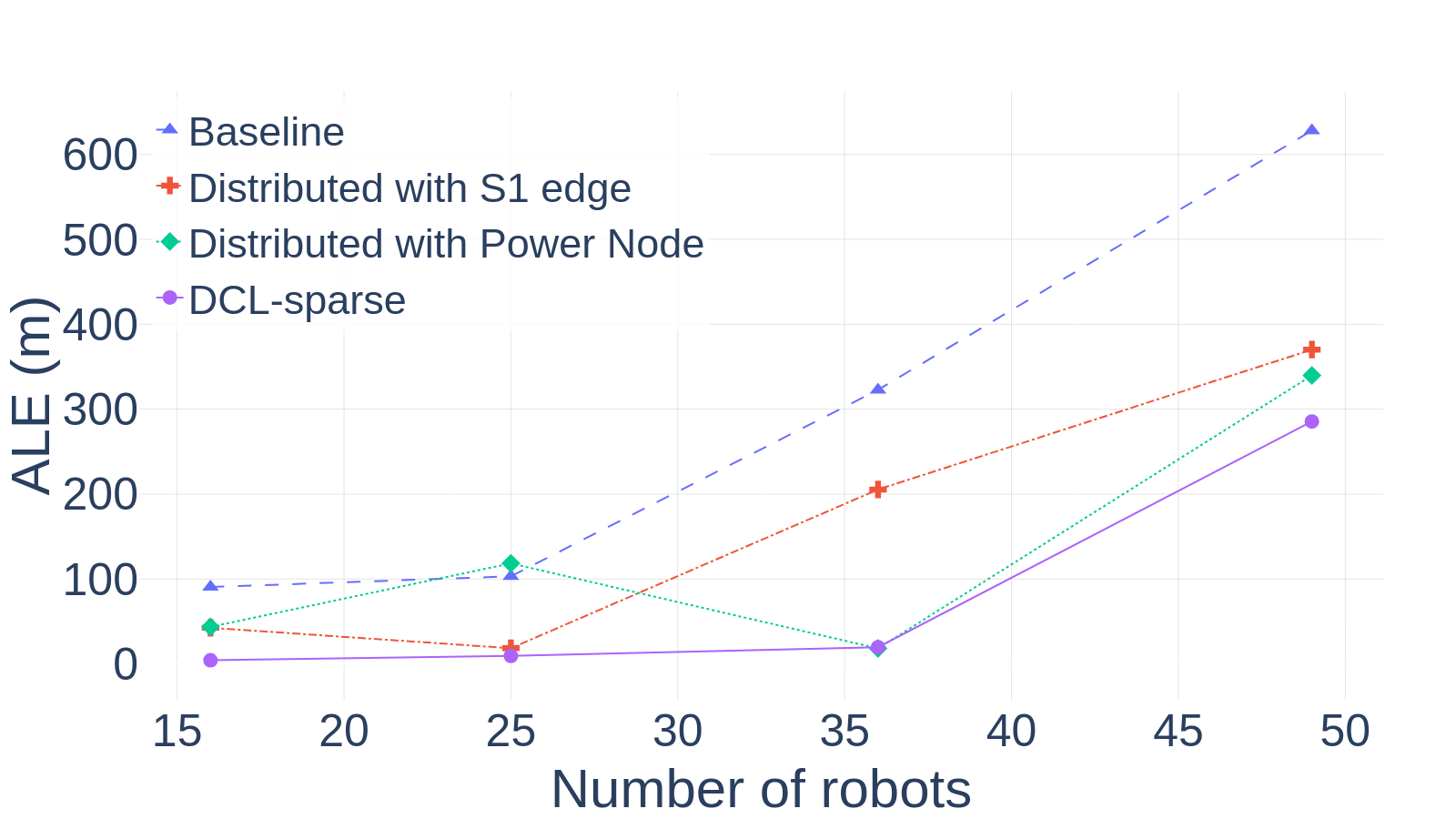} \\
            \footnotesize  (c) Noise impact &  \footnotesize  (d) Number of robots \\
        \end{tabular}
    }
    \caption{Impact of the robot configuration: square shape (a) and random shape (b) (with n=25 and $R(ggr) \leq 0.67$), sensing noise in sparse graphs with radius set to 0.4 ($<ggr$) (c), and the scalability with number of robots (d). }
    \label{fig: experiment}
    \vspace{-4mm}
\end{figure*}
\subsection{Performance, Scalability, and Formation Analysis}
We consider three scenarios in our experiments for rigorous analysis of the distributed localization algorithms.

\noindent \textbf{Qualitative Performance:} We present the formation results for different localization algorithms for a square-shaped initial configuration with a sensing radius set to exhibit a non-ggr graph. As shown in Fig.~\ref{fig: formation}, the baseline method suffers from both flipping and ambiguity. The S1-Edge algorithm also experiences flipping but performs better than the baseline. The emitter method exhibits flipping, primarily along the diagonal direction. In contrast, the DCL-sparse algorithm shows no flipping or ambiguity as it leverages the benefits of both the S1-Edge and the Emitter concepts.

\noindent \textbf{Impact of Graph Sparsity:} We assess the evolution of ALE under different graph rigidity conditions. Fig.~\ref{fig: sparsity} presents these results. Note that for radii where  \( 0.4 < \text{ggr} \)  and  \( 1.5 > \text{ggr} \), the DCL-Sparse approach demonstrates superior performance compared to other methods, providing better convergence (up to 95\% reduction in ALE). This is particularly notable when compared to the performance at radius \(0.4 \) (non-ggr sparse graph), where typical algorithms fail. The DCL-Sparse approach not only achieves optimal performance at this radius but also exhibits enhanced convergence across the broader range of radii.

\noindent \textbf{Impact of Robot Configuration:} We assess the ALE under different ground truths (initial configurations) and radius (sensing range). 
As the range of robots increases, the network becomes denser, providing more connections between robots. The DCL-sparse algorithm is designed to handle this increased connectivity efficiently. As shown in Fig.~\ref{fig: experiment}(a)-(b), the overall ALE for the DCL-sparse localization algorithm is lower than for the other methods in both configurations and converges to the minimum ALE faster as the range increases.

\noindent \textbf{Impact of Noise:} As shown in Fig.~\ref{fig: experiment}(c), adding Gaussian noise in the distance measurements (STD 0 to 0.3) results in only slight changes in the ALE for the proposed DCL-sparse localization algorithm (remains close to 0), whereas the baseline localization method experiences significant changes (from 150m to 225m). 
The baseline method is highly sensitive to noise, while the S1 Edge alone handled low-noise conditions well but suffered from increased noise levels due to its dependence on high-noise link measurements. The power node increases the sensing density, providing a balance between low ALE and consistent performance.
Thus, the DCL-sparse demonstrates superior robustness compared to the baseline localization method by leveraging the advantages of S1 Edge and Power node. 

\noindent \textbf{Scalability:} Fig.~\ref{fig: experiment}(d) shows that the baseline method struggles to scale with an increasing number of robots, showing significant degradation due to its reliance only on the local neighborhoods, and the ALE is compounded with more robots in the team. For small systems ($n \leq 25$ robots), the Power Node retains competitive performance but deteriorates beyond 35 robots. DCL-sparse algorithm demonstrates superior scalability as it maintains lower ALE even for a larger number of robots than other methods because it can better handle network density and computational challenges that typically increase after reaching a saturation point.

\par

\section{Conclusion}
This paper introduces a distributed multilayered approach for achieving distributed cooperation localization for MRS with sparse sensing graphs. It incorporates shadow edge techniques and UAV deployment to enhance robot swarm accuracy in GPS-denied environments. The DCL-Sparse, validated through rigorous theoretical and experimental analysis, outperformed existing methods in reducing localization across different graph sparsity settings, initial robot configurations, and sensor noise levels. The algorithm scales effectively with more robots, and the optimal UAV deployment enhances sensing coverage, further reducing errors.



\section*{Acknowledgements}
This research was sponsored by the Army Research Laboratory and was accomplished under Cooperative Agreement Number W911NF-17-2-0181. 
The views and conclusions contained in this document are those of the authors and should not be interpreted as representing the official policies, either expressed or implied, of the Army Research Laboratory or the U.S. Government. The U.S. Government is authorized to reproduce and distribute reprints for Government purposes notwithstanding any copyright notation herein.

\bibliographystyle{IEEEtran}
\bibliography{ref}

\end{document}